\documentclass[10pt,journal,compsoc]{IEEEtran}
\usepackage{amsfonts,amsmath,amssymb,amsthm,bm}
\usepackage{graphicx,psfrag,epsf}
\usepackage{enumerate}
\usepackage{adjustbox}
\usepackage{dsfont}
\usepackage{comment}
\usepackage{xcolor}
\usepackage{hyperref}

\newcommand{\E}{\mathbb{E}}

\theoremstyle{plain}
\newtheorem{definition}{Definition}
\newtheorem{theorem}{Theorem}
\newtheorem{cor}{Corollary}
\newtheorem{remark}{Remark}

\newtheorem*{theoremLugosi}{Theorem 2.2, 2.3 in \cite{lugosi1992learning}}

\ifCLASSOPTIONcompsoc
  \usepackage[nocompress]{cite}
\else
  \usepackage{cite}
\fi

\ifCLASSINFOpdf

\else
 
\fi

\begin{document}

\title{Deep Learning is Provably Robust to Symmetric Label Noise}

\author{Carey E.~Priebe,~Ningyuan~(Teresa)~Huang,~Soledad Villar,~Cong~Mu, Li~Chen
\IEEEcompsocitemizethanks{\IEEEcompsocthanksitem Carey E.\ Priebe
is Professor in the Department of Applied Mathematics and Statistics (AMS), the Center for Imaging Science (CIS), and the Mathematical Institute for Data Science (MINDS), Johns Hopkins University. E-mail: cep@jhu.edu 
\IEEEcompsocthanksitem Ningyuan (Teresa) Huang and Cong Mu are PhD students in the Department of Applied Mathematics and Statistics, Johns Hopkins University. E-mail: nhuang19@jhu.edu, cmu2@jhu.edu
\IEEEcompsocthanksitem Soledad Villar is an Assistant Professor in the Department of Applied Mathematics and Statistics, and the Mathematical Institute for Data Science (MINDS), Johns Hopkins University. E-mail: svillar3@jhu.edu, 
\IEEEcompsocthanksitem Li Chen is Research Scientist with Meta AI. E-mail: lichen66@fb.com \protect
}
}

\markboth{IEEE Transactions on Pattern Analysis and Machine Intelligence (Submitted)}%
{Shell \MakeLowercase{\textit{et al.}}: Bare Demo of IEEEtran.cls for Computer Society Journals}

\IEEEtitleabstractindextext{%
\begin{abstract} 
Deep neural networks (DNNs) are capable of perfectly fitting the training data, including memorizing noisy data. It is commonly believed that memorization hurts generalization. Therefore, many recent works propose mitigation strategies to avoid noisy data or correct memorization. In this work, we step back and ask the question: Can deep learning be robust against massive label noise without any mitigation? We provide an affirmative answer for the case of symmetric label noise: We find that certain DNNs, including under-parameterized and over-parameterized models, can tolerate massive symmetric label noise up to the information-theoretic threshold. By appealing to classical statistical theory and universal consistency of DNNs, we prove that for multiclass classification, $L_1$-consistent DNN classifiers trained under symmetric label noise can achieve Bayes optimality asymptotically if the label noise probability is less than $\frac{K-1}{K}$, where $K \ge 2$ is the number of classes. Our results show that for symmetric label noise, no mitigation is necessary for $L_1$-consistent estimators. We conjecture that for general label noise, mitigation strategies that make use of the noisy data will outperform those that ignore the noisy data.
\end{abstract}



}

\maketitle

\IEEEdisplaynontitleabstractindextext

%
\IEEEpeerreviewmaketitle

\IEEEraisesectionheading{\section{Introduction}\label{sec:introduction}}

\IEEEPARstart{C}{onsider} the classical classification setup \cite[p2]{dgl}: Let $(X,Y), (X_1,Y_1), \cdots (X_n,Y_n) \overset{iid}{\sim} F_{XY}$,
where feature vector $X$ lives in $\Re^d$ and class label $Y$ lives in $[K] = \{1,\cdots,K\}$.
Denote the training data by $\mathcal{T}_n = \{(X_1,Y_1), \cdots (X_n,Y_n)\}$.
Our goal is to learn a classifier $g: \Re^d \times (\Re^d \times [K])^n \to [K]$ using $\mathcal{T}_n$
  to predict the true but unobserved class label $Y$ based on the observed test feature vector $X$.
Performance is measured by the conditional probability of error,
\begin{equation}
    L(g) = \mathbb{P}[g(X;\mathcal{T}_n) \neq Y|\mathcal{T}_n]. \label{eqn:error}
\end{equation}

Now consider the setting wherein we do not observe the $Y_i$ but rather noisy labels $Z_i$.
For $P_i \in [0,1]$,
 let noisy class label $Z_i$ be given by
 $\mathbb{P}[Z_i = Y_i] = 1-P_i$
 and $Z_i$ distributed on $[K] \setminus \{Y_i\}$ with probability $P_{i}$;
 $P_i = 0$ means no noise in the $Z_i$ and $P_i = (K-1)/K$ means no information in the noisy labels $Z_i$. Common label noise structures include class-dependent noise and incident-dependent noise. Class-dependent noise assumes $P_i$ is the same for all instances in the same class, which can be modeled by a noise transition matrix $A \in \Re^{K \times K}$, where $\mathbb{P}[Z_i = l \mid Y_i = k] = A_{kl}$; Symmetric label noise further assumes that $A$ is symmetric with diagonal entries as $1- \alpha$, off-diagonal entries as $\alpha/(K-1)$.

Thus, we have
$(X_i,Y_i,Z_i,P_i) \overset{iid}{\sim} F_{X,Y,Z,P}$.
Again: $X$ is the feature vector and $Y$ is the true class label;
now $Z$ is the noisy class label and $P$ characterizes the label noise. The classifier $g$ is trained on the noisy dataset $\tilde{\mathcal{T}}_n = \{(X_1,Z_1), \cdots (X_n,Z_n)\}$, and evaluated on the clean sample
\begin{equation}
    L(\tilde{g}) = \mathbb{P}[g(X;\tilde{\mathcal{T}}_n) \neq Y|\tilde{\mathcal{T}}_n]. \label{eqn:error}
\end{equation}

It is well known that the optimal classifier is given by the Bayes decision rule:
\begin{equation}
    g^* (x) = \arg \max_{1\le k \le K} \mathbb{P}[Y = k | X=x],
\end{equation}
with the Bayes error given by
\begin{equation}
  L(g^*) = \mathbb{P}[g^*(x) \ne Y] = 1 - \mathbb E [\max_k p_k(X)],
\end{equation}
where $p_k(x) = \mathbb{P}[Y = k | X=x]$ for $k \in [K]$ denotes the a posteriori probabilities.

One natural decision rule is to approximate the a posteriori probability given the training data. In the non-noisy setting, it is well known that if the posterior estimates are $L_1$ (or $L_2$) consistent, then the plug-in Bayes classifier (that maximizes the a posteriori probabilities) is consistent \cite[Section 2.5]{dgl}. However, in the noisy label dataset, one can only hope to estimate the noisy label posterior $q_k(x) = \mathbb{P}[Z = k | X=x] \, (k \in [K])$ via empirical distribution $q_{kn}(x)$. Consider the plug-in classifier again but from noisy $\tilde{\mathcal{T}}_n$,
\begin{equation}
    \tilde{g}_n (x) = \arg \max_k q_{kn}(x). \label{eqn:consistent-g}
\end{equation}

If the posterior estimates are $L_1$-consistent yet for the noisy label posterior $q_k(x)$, how well does the noisy plug-in classifier $\tilde{g}_n(x)$ compared to the Bayes optimal classifier? Remarkably, for binary classification with symmetric label noise, the Bayes decision rule based on noisy posterior $q_k(x)$ remains the same as that of the clean posterior $p_k(x)$ up to the information-theoretic threshold $P_i = 1/2$ \cite{lugosi1992learning, natarajan2013learning, menon2015learning}. Thus, if $q_{kn}(x)$ is a $L_1$-consistent estimator of $q_k(x)$, then $\tilde{g}_n(x)$ yields Bayes-optimal performance asymptotically \cite{lugosi1992learning}.

We now turn to deep neural network classifiers (DNNs) and ask the same question: How well does the noisy plug-in DNN compared to the Bayes optimal classifier? In other words, can DNNs be robust against massive label noise while using noisy posteriors without any mitigation? Empirically, DNNs can memorize arbitrary noisy labels during training and may generalize poorly ~\cite{zhang2021understanding,arpit2017closer}. This phenomenon motivates many follow-up works to design robust deep learning models by mitigating the effect of label noise, including model-free methods that do not explicitly model the noise structure, and model-based methods that assume or estimate the label noise structure (see \cite{algan2021image} for a recent survey).   

\vspace{1em}
\noindent
\textbf{Related work. }
In the model-free literature, recent theoretical results show that imposing regularization on DNNs, such as early stopping \cite{pmlr-v108-li20j} or weight regularizations \cite{liu2020earlylearning, xia2021robust, arpit2017closer}, constrains the model to ignore noisy labels during gradient updates and thus mitigate the effect of label noise. More precisely, \cite{pmlr-v108-li20j} showed that, with early stopping, one hidden-layer fully-connected neural network is robust to label noise up to $ \frac{1}{4(K-1)}$ class-dependent noise probability\footnote{Assume that $K$-class labels lie in $[-1,1]$ and labels from different classes have Euclidean distance at least $\delta$ (i.e., $\delta \le \frac{2}{K-1}$), Theorem 2.2 in \cite{pmlr-v108-li20j} proves robustness up to  noise probability $\frac{\delta}{8} \le \frac{1}{4(K-1)}$.}. Their analysis relies on the key assumptions that the Jacobian of the network has a low-rank structure, which implies the network ``fits the correct labels essentially ignoring the noisy labels.'' as stated in \cite{pmlr-v108-li20j}. However, they conjecture that the tolerance bound can be improved up to the order of $n$ noisy labels.   Similarly, \cite{liu2020earlylearning} observed that ``...early in training, the gradients corresponding to the correctly labeled examples dominate the dynamics---leading to early progress towards the true optimum---but that the gradients corresponding to wrong labels soon become dominant'' and proposed regularization to prevent memorization of noisy labels. 

In the model-based literature, the most relevant work is \cite{patrini2017making}, which shows that by performing loss correction, DNNs can tolerate label noise as long as the noise transition matrix $A$ is invertible (i.e., tolerance threshold up to $\frac{K-1}{K}$ for symmetric noise); Such tighter bound compared to \cite{pmlr-v108-li20j} is obtained with the extra assumption that the label noise is known or can be perfectly estimated from the data.

\vspace{1em}
\noindent
\textbf{Our contribution} 
In this paper, we show that when the symmetric label noise is bounded by $\frac{K-1}{K}$, DNNs trained with noisy data can achieve Bayes optimal performance asymptotically, \textit{without the need for any label noise mitigation}. 
The key observation is that DNNs are universally consistent \cite{farago1993strong, lin2021universal, Drews2022} and thus $L_1$-consistent. This allows us to make use of a generalized version of results in \cite{lugosi1992learning}, extending from the binary setting to the multiclass setting for symmetric label noise. 
We answer the conjecture in \cite{pmlr-v108-li20j} affirmatively in the special setting of symmetric label noise, without requiring the restrictive assumption in \cite{patrini2017making} to perfectly estimate the noise structure. Our results also hold for other $L_1$-consistent estimators, which may be of independent interest.


\section{Main Results} \label{sec:stat}

To prove DNNs trained from symmetric noisy labels can achieve Bayes optimality asymptotically, we first generalize the characterization in \cite[Theorem 2.3]{lugosi1992learning} for binary classification to multiclass classification. We then proceed to show that DNNs are $L_1$-consistent estimator of the (noisy) posteriors based on the universal consistency results of DNNs from \cite{farago1993strong, lin2021universal, Drews2022}. 

To present our main results, we recall the following definitions and key results from \cite{lugosi1992learning}.

\begin{definition}[Consistency]\label{defn:consistent}
Consider the setup introduced in Section \ref{sec:introduction}. A sequence of posterior estimates $\{q_{kn}\}$ is called $L_1$-consistent for a certain distribution $F_{XY}$ if
\begin{equation}
\lim_{n \to \infty} \E( \sum_{k=1}^K |q_{kn}(X) - q_k(X)|) = 0.    \label{eqn: L1-consistency}
\end{equation}

It is called $L_2$-consistent for a certain distribution $F_{XY}$ if
\begin{equation}
\lim_{n \to \infty} \E( \sum_{k=1}^K (q_{kn}(X) - q_k(X))^2)  = 0.    \label{eqn: L1-consistency}
\end{equation}

Universal consistency requires consistency to hold for all distributions $F_{XY}$ with $\E(Y^2) < \infty$.
\end{definition}






\begin{theoremLugosi} \label{thm:lugosi}
Consider the binary classification setting, where $\alpha, \beta$ denote the label noise probability for class $0, 1$ respectively. Let the classifier $\tilde{g}_n(x)$ be defined as \eqref{eqn:consistent-g}, which uses maximizing a posteriori (MAP) decision rule on a $L_1$-consistent estimator $q_{kn}(X)$. Assume $\max(\alpha,\beta) < 1/2$. 
Asymptotically, if $\alpha, \beta$ are known, then 
\begin{equation}
 L(\tilde{g}_n) \to L (g^*); \label{eqn:lugosi_known}
\end{equation}

If $\alpha, \beta$ are unknown, then 
\begin{equation}
 L(\tilde{g}_n) \to L (g^*) \left[1+\frac{2|\alpha-\beta|}{1-2 \max (\alpha, \beta)}\right].      \label{eqn:lugosi_unknown}
\end{equation}

\end{theoremLugosi}

In practice, $\alpha, \beta$ are typically unknown. Yet for symmetric label noise (i.e., $\alpha = \beta$), $\tilde{g}_n(x)$ is asymptotically Bayes-optimal until the noise probability exceeds $0.5$. On the other hand, for class-dependent label noise, higher asymmetry implies worse performance --- a constant times the Bayes risk. Therein, we refer to the maximum label noise threshold that preserves Bayes optimality as the statistical limit.

We are ready to present our main results, which extend the binary setting in \cite{lugosi1992learning} to the multiclass setting for symmetric label noise.
\begin{theorem} \label{thm:main}
Consider the multiclass classification setting with $K \ge 2$ classes and symmetric label noise with noise probability $\alpha$. Let the classifier $\tilde{g}_n(x)$ be defined as \eqref{eqn:consistent-g} which uses MAP on a $L_1$-consistent estimator $q_{kn}(X)$. If $\alpha < \frac{K-1}{K}$, then as $n \to \infty$, for both known and unknown $\alpha$,
$$
L(\tilde{g}_n) \to  L (g^*).
$$
\end{theorem}

\begin{proof}

Let $\alpha$ denote the noisy label probability (i.e., $P[Z_i = Y_i] = 1 - \alpha$). Observe that the symmetric noise transition matrix is given by \begin{equation}
A_{ij} = \begin{cases}
     1-\alpha,& \text{if } i=j\\
     \frac{\alpha}{K-1},  & \text{otherwise.} \label{eqn:sym_noise}
\end{cases}    
\end{equation}

In the case where $\alpha$ is known (and thus $A$ is known), observe that the noisy posteriors and the true posteriors are related by
\begin{equation}
  [q_{1}(x), \cdots, q_K(x)] = [p_1(x), \cdots, p_K(x)] A.   \label{eqn:transition}
\end{equation}
Therefore, the invertibility of $A$ yields sufficient and necessary condition for estimating the true posteriors $p_k(x)$ from noisy posteriors $q_k(x)$ and thus obtaining the Bayes optimal decision. Further observe that for symmetric label noise,
$$
A = \frac{\alpha}{K-1} \mathbf{1}_{K \times K} + (1-\alpha - \frac{\alpha}{K-1}) I,
$$
where $ \mathbf{1}_{K \times K}$ denotes the all-ones matrix in $\Re^{K \times K}$. Thus, $A$ is invertible if and only if $1-\alpha - \frac{\alpha}{K-1} > 0$. In other words, the noisy plug-in classifier can tolerate label noise up to the breakdown point at $\alpha = \frac{K-1}{K}$. When $K=2$, we recover eqn \eqref{eqn:lugosi_known} in \cite[Theorem 2.2]{lugosi1992learning}.

In the case where $\alpha$ and thus $A$ are unknown (while the form of $A$ is known as eqn \eqref{eqn:sym_noise}), we can write the true posterior as a function of the noisy posterior using eqn \eqref{eqn:transition}, 
\begin{align}
    p_k(x) &= (1- \alpha - \frac{\alpha}{K-1})^{-1} (q_k(x) - \frac{\alpha}{K-1}) . \label{eqn:mono}
\end{align}
When $\alpha < \frac{K-1}{K}$, the coefficient $(1- \alpha - \frac{\alpha}{K-1})^{-1} > 0$ and so $p_k(x)$ is monotonically increasing with $q_k(x)$. Therefore, by monoticity, if we know the noisy posteriors such that $q_1(x) \ge \ldots \ge q_K(x)$, then $p_1(x) \ge \ldots \ge p_K(x)$. In other words, the noisy decision coincides with the Bayes decision $\arg \max_{k \in [K]} q_k(x) = \arg \max_{k \in [K]} p_k(x)$. Now, since the classifier $\tilde{g}_n(x)$ is a $L_1$-consistent estimator, then the empirical noisy posterior $q_{kn}(x) \to q_k(x)$ when $n \to \infty$, so we can estimate the noisy posterior perfectly in the asymptotic limit, and obtain the Bayes optimal performance.
\end{proof}

\begin{remark} \label{rem1}
Even when the noise probability is unknown, symmetric label noise (up to the information-theoretic threshold) effectively maintains the ordering of the true posteriors, and therefore leads to Bayes optimality based on the noisy posteriors. However, class-dependent label noise typically leads to sub-optimality, as shown in eqn \eqref{eqn:lugosi_unknown} for binary classification and further discussed in \cite{scott2013classification}. A natural mitigation strategy relies on estimating $A$ from data \cite{patrini2017making, menon2015learning}: if $A$ can be perfectly recovered from data, then it is possible to achieve Bayes optimality for unknown, class-dependent label noise, as shown in \cite[Thm 3]{patrini2017making}.
\end{remark}

\begin{remark} \label{rem2}
Theorem \ref{thm:main} is applicable for any $L_1$-consistent estimator. For example, Adaboost is universally consistent when using appropriate early-stopping and sufficiently rich base learners \cite{bartlett2006adaboost}. Therefore it can tolerate massive symmetric label noise. However, Adaboost without consistency guarantees is highly susceptible to symmetric label noise \cite{long2008random}.
\end{remark}

\begin{remark} \label{rem3}
Although $L_1$-consistency is sufficient to derive robustness against label noise, it is not necessary. For example, \cite{ghosh2017robustness} show that decision tree based on Gini impurity splitting can achieve the label noise tolerance up to the statistical limit, while such estimator is not universally consistent \cite[P338]{dgl}.
\end{remark}

It remains to show that DNNs are $L_1$-consistent. Observe that $L_2$-consistency implies $L_1$-consistency \cite[DGL Cor 6.2]{dgl}, since for each $k \in [K]$,
\begin{align}
 \E(|q_{kn}(X) - q_k(X)|) &= \int_{\Re^d} |q_{kn}(x) - q_k(x)| \mu(dx) \nonumber \\
 & \le \Big( \int_{\Re^d} |q_{kn}(x) - q_k(x)|^2 \mu(dx) \Big)^{1/2}. \label{eqn:L2-consistency}  
\end{align}

Thus, our results are immediate from the ($L_2$) universal consistency results of DNNs from  \cite{farago1993strong, lin2021universal, Drews2022}. More precisely, universal consistency of under-parameterized neural networks was established in \cite{farago1993strong, barron1994approximation} for fully-connected neural networks and \cite{lin2021universal} for convolutional neural networks (CNNs), whereby \textit{under-parameterized} we mean that in the asymptotic limit, the ratio of the number of parameters of the DNN and the number of data samples is less than $1$. This is in contrast to the \textit{over-parameterized} networks where such ratio is greater than $1$. Remarkably, \cite{Drews2022} recently show that even over-parameterized networks can also be universally consistent, given proper setup in the gradient descent optimization (e.g., initialization, step size, and the number of iterations). To conclude, we establish the following:
\begin{cor}
Consider the $K$-class classification setting in Theorem \ref{thm:main} where the classifier $\tilde{g}_n(x)$ is a $L_1$-consistent deep neural network (DNN). If the symmetric noise probability $\alpha < \frac{K-1}{K}$, then such DNN trained from noisy data without mitigation can achieve Bayes optimality asymptotically.
\end{cor}






\section{Numerical Evidence}
To demonstrate our results, we conduct numerical simulations on training CNNs on noisy benchmark datasets (see Appendix \ref{app} for full details). As shown in Figure \ref{fig:gap}, when training with symmetric label noise, the classification performance degrades very slowly until the statistical limit $(K-1)/K$ (yellow dotted line), whereas the tolerance bound $1/4(K-1)$ (grey dotted line) in \cite{pmlr-v108-li20j} is much looser. Similar empirical evidence can be found in \cite{markermap2022} that shows variational auto-encoder classifiers are robust to symmetric label noise up to the statistical limit. 

\begin{figure}[htb!]
\centering
\includegraphics[width=7cm]{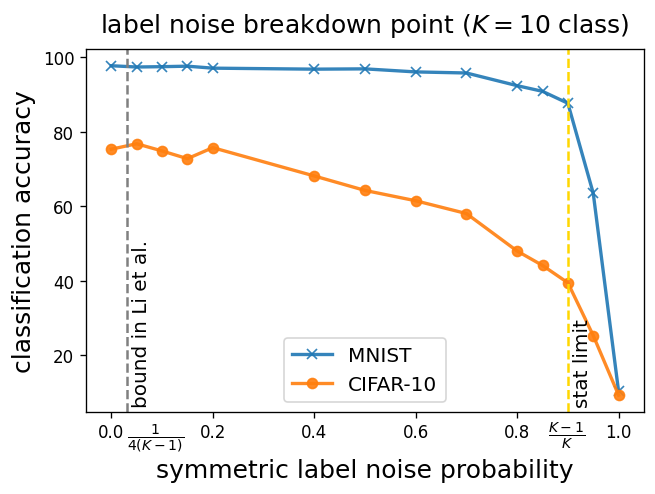}
\caption{CNNs trained with \textit{symmetric} label noise on MNIST and CIFAR10 datasets.
Experimental results agree with the statistical analysis, and demonstrate deep learning models can be surprisingly robust against massive symmetric label noise.}
\label{fig:gap}
\end{figure}

As discussed in Remark \ref{rem1}, class-dependent label noise can be more harmful than symmetric label noise. We illustrate such phenomenon in Figure \ref{fig:gap_asym}, where the class-dependent noise transition matrix is given by
\begin{equation}
A_{ij} = \begin{cases}
     1-\alpha,& \text{if } i=j\\
     \alpha,  & \text{if } j=(i+1)\;\mathrm{mod}\;10 \\ 
     0 & \text{otherwise.} \label{eqn:asym_noise}
\end{cases}    
\end{equation}

Note that each row of $A$ in \eqref{eqn:asym_noise} only has two nonzero entries, and thus such class-dependent noise effectively reduces the multiclass problem to the binary setting (conditional on each class). Yet when $K=2$, the statistical limit $1/2$ is still more optimistic than the tolerance bound $1/4$ in \cite{pmlr-v108-li20j}, and achievable as shown in Figure \ref{fig:gap_asym}.

\begin{figure}[htb!]
\centering
\includegraphics[width=7cm]{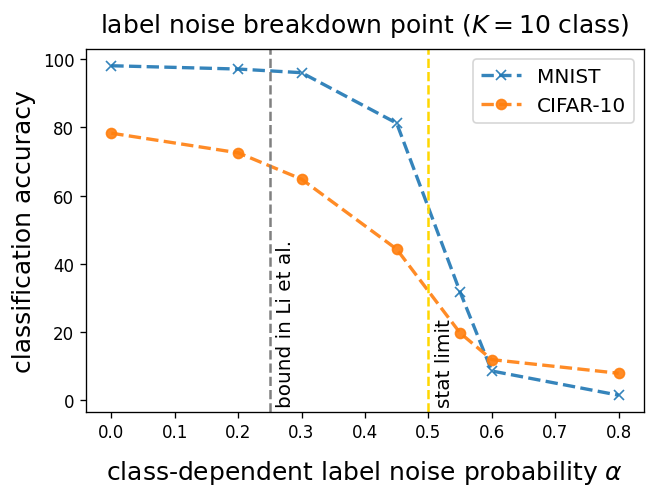}
\caption{CNNs trained with \textit{class-dependent} label noise on MNIST and CIFAR10 datasets. Here, each class label $k \in \{1, \ldots, 10\}$ is flipped to class $(k+1)\;\mathrm{mod}\;10$ with probability $\alpha$, and remains unchanged with probability $1-\alpha$. This class-dependent label noise structure effectively reduces the multiclass setting to the binary setting. Yet our statistical limit is still achievable and tighter than the bound in \cite{pmlr-v108-li20j}.}
\label{fig:gap_asym}
\end{figure}

\section{Discussion}

This short note establishes the statistical limits of deep learning classifiers trained with label noise: Deep neural networks can be surprisingly robust against symmetric label noise without mitigation. Such robustness guarantees hold for any $L_1$-consistent DNN, including both under-parameterized and over-parameterized models. Empirical simulations confirm that the statistical limit is achievable. 

We hope that the statistical limit might provide an impetus for efforts to understand deep learning against label noise. One interesting direction is to investigate whether we can relax the $L_1$-consistency necessary condition. Our numerical experiments suggest this is plausible ($L_1$-consistency was not enforced in the models), and Remark \ref{rem3} points out a potential path by connecting ReLU-based DNNs to partition-based methods such as decision trees. 

In future work, we aim to study the statistical limit under general label noise structure, including class-dependent and incident-dependent noise. Based on our current results, we conjecture that mitigation strategies that make use of the noisy data, such as using them to estimate the noise structure, will outperform those that ignore the noisy data.

\section*{Acknowledgements}

The authors thank George A Kevrekidis, Joshua Agterberg, and Youngser Park for their valuable comments on the paper. Cong Mu and Teresa Huang are partially supported by the Johns Hopkins Mathematical Institute for Data Science (MINDS) Data Science Fellowship. Soledad Villar is supported by NSF DMS 2044349, EOARD FA9550-18-1-7007, and NSF-Simons MoDL (NSF DMS 2031985).

\ifCLASSOPTIONcaptionsoff
  \newpage
\fi



%

\bibliographystyle{IEEEtran}

\bibliography{sample}

\clearpage

\appendices

\section{Experiment Set-up}  \label{app}

To empirically verfiy the tightness of the statistical limit, we use the following standard datasets and convolutional neural network architectures for image classification with symmetric label noise.

\hspace{0.5em}

\noindent
\textbf{MNIST}~\cite{lecun1998gradient}. 
\begin{itemize}
    \item Data: The collection of grey-scale handwritten digit images (10-class) of size 28x28, with a training set of 60000 examples and a test set of 10000 examples.
    \item Architecture: CNNs with two convolution layers followed by two fully connected layers.
\end{itemize}

\noindent
\textbf{CIFAR10}~\cite{krizhevsky2009learning}.
\begin{itemize}
    \item Data: Labeled subset (10-class) of the 80 million color images of size 32x32, with a training set of 50000 examples and a test set of 10000 examples.
    \item Architecture: CNNs with three convolutional blocks (each consists of two convolutional layers and one pooling layer) followed by three fully connected layers. 
\end{itemize}

\noindent
\textbf{Label noise.} We use the following sequence of label noise probabilities:
\begin{itemize}
    \item Symmetric: 
$\{0, 0.05, 0.1, 0.15, 0.2, 0.4, 0.5, 0.6, 0.7, \\ 0.8, 0.85, 0.9, 0.95, 1\}$;
  \item Class-dependent: 
$\{0,0.2,0.3,0.45,0.55,0.6,0.8\}.$
\end{itemize}

\noindent
\textbf{Training and Evaluation.} 
The transformation for the input images follow the recommended steps for pre-trained models in \texttt{PyTorch}~\cite{paszke2019pytorch}. For all experiments, we use  stochastic gradient descent with batch size of 64, learning rate of 0.01. For MNIST, the networks are trained for 3 epochs using momentum 0.5; For CIFAR10, the networks are trained for 10 epochs using momentum 0.99. We evaluate the trained models on the original test dataset without label noise.

\end{document}